\documentclass[runningheads]{llncs}
\usepackage[T1]{fontenc}
\usepackage{latexsym}
\usepackage{amssymb}
\usepackage[nosumlimits]{amsmath}
\usepackage{graphicx}
\usepackage{xcolor}
\usepackage[pdftex,
              pdftitle={Comments on Friedman's Method for Class Distribution Estimation},
               pdfsubject={Machine learning},
               pdfkeywords={},
               pdfauthor={Dirk Tasche},
               pdfstartview=FitR,
               breaklinks=true,
               colorlinks=true,
               citecolor=teal,
               linkcolor=magenta,
               urlcolor=olive]{hyperref}
               
\begin{document}
\title{Comments on Friedman's Method for Class Distribution Estimation}
\titlerunning{Comments on Friedman's Method}
%
\author{Dirk Tasche\orcidID{0000-0002-2750-2970}}
\authorrunning{D.~Tasche}
%
\institute{Unit for Data Science and Computing, North-West University, South Africa\\
\email{55801447@nwu.ac.za}}
\maketitle              
\begin{abstract}
The purpose of class distribution estimation (also known as quantification)
is to determine the values of the prior class probabilities in 
a test dataset without class label observations. A variety of methods
to achieve this have been proposed in the literature, most of them based on the assumption
that the distributions of the training and test data are related through
prior probability shift (also known as label shift). Among these methods, 
Friedman's method has recently been found to perform relatively well
both for binary and multi-class quantification. 
We discuss the properties of Friedman's method and another approach
mentioned by Friedman (called DeBias method in the literature) in the 
context of a general framework for designing linear equation systems
for class distribution estimation.

\end{abstract}
%
\section{Introduction}

The purpose of class distribution estimation (also known as quantification)
is to determine the values of the prior class probabilities in 
a test dataset without class label observations. A variety of methods
to achieve this have been proposed in the literature, most of them based on the assumption
that the distributions of the training and test data are related through
prior probability shift (also known as label shift). See Gonz\'{a}lez et al.~\cite{Gonzalez:2017:RQL:3145473.3117807}
and Esuli et al.~\cite{esuli2023learning} for recent surveys of applications of and methods 
for quantification.

Friedman's \cite{friedman2014class} method has recently been found to perform relatively well
both for binary and multi-class quantification (Schuhmacher et al.~\cite{schumacher2021comparative}, 
Donyavi et al.~\cite{donyavi2023mc}). On many real-world datasets, the performance
of Friedman's method seems to exceed the performance of the EM algorithm (Saerens et al.~\cite{saerens2002adjusting})
which is an implementation of the maximum likelihood estimator for the test prior class probabilities (also called
class prevalences). This observation is somewhat surprising because both Friedman's estimator 
and the EM algorithm involve estimates of the training posterior class probabilities which
are notoriously hard to estimate. Hence one might expect that the performances of Friedman's method
and the EM algorithm are at a more comparable level.

In order to find an explanation for the relatively good performance of Friedman's method, we study its properties and 
the properties of another approach
mentioned by Friedman (called DeBias method by Casta{\~n}o et al.~\cite{castano2024quantificationlib}) in the 
context of a general framework for designing linear equation systems
for class distribution estimation.

The outline of this paper and its contributions to the literature are as follows:
\begin{itemize}
\item Section~\ref{se:setting} sets out the general assumptions and notation for the rest of the paper.
\item In Section~\ref{se:linear}, we discuss the general framework for designing linear equation systems
for class distribution estimation. \eqref{eq:matrix} and \eqref{eq:covmatrix} of Theorem~\ref{th:basic} below
appear to be novel, covariance-based versions of the basic equation \eqref{eq:Firat}.
\item In Section~\ref{se:method}, we describe Friedman's method in detail and propose
an alternative implementation that avoids direct estimation of the posterior class probabilities 
(Remark~\ref{rm:alternative} below). 
\item In Section~\ref{se:unique}, we investigate conditions for the uniqueness of the solutions 
to linear equation systems for class distribution estimation. In Remark~\ref{rm:debias}, we show that 
DeBias, the second
method proposed for the binary case by Friedman~\cite{friedman2014class} which involves the variance of one
of the posterior class probabilities,
is a special case of a covariance matrix-based approach to the multi-class case considered in 
Corollary~\ref{co:solved}. This provides an answer to the open research question
``How to generalise the inequality of Corollary 6 of Tasche~\cite{tasche2014exact} to the multi-class case?''
raised in Section~4.12 of Krempl et al.~\cite{krempl_et_al:DagRep.10.4.1}. In addition, we show
that the population versions of DeBias and `Probabilistic adjusted count (PAC)' 
by Bella et al.~\cite{bella2010quantification} are identical (Remark~\ref{rm:same} below).
\item In Section~\ref{se:example}, we compare the asymptotic variances of DeBias, Friedman's method
and the maximum likelihood estimator in the binary case by means of a numerical example. The setting
of the example is semi-asymptotic with an infinite training dataset and a finite large test dataset.
\item Section~\ref{se:conclusions} concludes the paper with a summary of the findings.
\end{itemize}

\section{Setting}
\label{se:setting}

For this paper, we assume the following setting which is quite common in the study of dataset shift 
(see, for instance, Moreno-Torres et al.~\cite{MorenoTorres2012521}):
\begin{itemize}
\item A class variable $Y$ with values in $\mathcal{Y} = \{1, \ldots, \ell\}$ with $\ell \ge 2$ (multi-class case). 
A features vector $X$ with values in $\mathcal{X}$. 
\item Each example (or instance) has a class label $Y$ and features $X$. 
\item In the training dataset, for all examples their features $X$ and labels $Y$ are observed. $P$ denotes the
training (joint) distribution, also called source distribution, of $(X, Y)$ 
of which the training dataset has been sampled.
\item In the test dataset, only the features $X$ of an example can immediately be observed. Its class label $Y$
becomes known only with delay or not at all. 
$Q$ denotes the test (joint) distribution, also called target distribution, of $(X, Y)$ of 
which the test dataset has been sampled.
\item We assume $0 < P[Y=y] < 1$, $0 < Q[Y=y] < 1$ for all $y\in \mathcal{Y}$.
\item For the sake of a more concise notation, we define $p_y = P[Y=y]$ and $q_y = Q[Y=y]$ for 
$y \in \mathcal{Y}$.
\end{itemize}

We also use the notation $E_P[Z] = \int Z\, d P$ and $E_Q[Z] = \int Z \,d Q$ 
for integrable real-valued random variables $Z$.

The setting described above is called \emph{dataset shift} or \emph{distribution shift} in the literature 
if training and test distribution are not the same, i.e.~$P \neq Q$. In the rest of the paper, we consider
the following more specific type of dataset shift.

\begin{definition}\label{de:priorShift} The training distribution $P$ and the test distribution $Q$ are related
through \emph{prior probability shift} if for all $y\in\mathcal{Y}$ and all measurable sets $M \subset \mathcal{X}$ 
it holds that\footnote{Recall the notion of conditional probability for events $A$ and $B$: 
$P[A|B] = \frac{P[A\cap B]}{P[B]}$ if $P[B]>0$ and $P[A|B] =0$ otherwise.}
\begin{equation*}
P[X\in M|Y=y] \ = \ Q[X\in M|Y=y].
\end{equation*}
\end{definition}
The term `prior probability shift' appears to have been coined by Storkey~\cite{storkey2009training}. 
In the literature, prior probability shift is also called \emph{target shift} (Zhang et 
al.~\cite{Zhang:2013:TargetShift}), \emph{label shift} (Lipton et al.~\cite{pmlr-v80-lipton18a}), or \emph{global
drift} (Hofer and Krempl~\cite{hofer2013drift}).

Prior probability shift implies dataset shift, i.e.~$P \neq Q$, if $P[Y=y] \neq Q[Y=y]$ for at least one
$y \in \mathcal{Y}$. Hence, as the class labels $Y$ are not observed in the test dataset, the test prior
probabilities $q_y = Q[Y=y]$ must be estimated from feature observations in the test dataset as well as 
feature and class label observations in the training dataset. Such an estimation procedure is called
\emph{quantification} or \emph{class distribution estimation}.

\section{Linear equations for class distribution estimation}
\label{se:linear}

In the following, we treat class distribution estimation under prior probability shift
as a parametric estimation problem in a family of mixture distributions:
\begin{itemize}
\item We consider the distributions $Q_X$ on $\mathcal{X}$ that can be represented as
\begin{equation}\label{eq:family}
Q_X[M] \ = \ \sum_{y=1}^\ell q_y\,P[X \in M | Y=y]
\end{equation}
for all measurable sets $M \subset \mathcal{X}$. The family of these distributions is
parame{\-}trised through the test prior class probabilities $(q_1, \ldots, q_\ell) \in (0,1)^\ell$
with the additional constraint 
\begin{equation}\label{eq:sumOne}
\sum_{y=1}^\ell q_y\ =\ 1.
\end{equation}
\item Unless stated otherwise, for the purposes of this paper we assume that the conditional 
feature distributions $P[X \in M | Y=y]$, $M \subset \mathcal{X}$, 
under the training distribution are known and do not contribute to the estimation uncertainty.
\item The parametrised distribution family defined in \eqref{eq:family} is identifiable in the sense
of Definition 11.2.2 of Casella and Berger~\cite{Casella&Berger}, i.e.~$Q_X$ and $Q_X^\ast$ differ whenever
the corresponding parametrisations $(q_1, \ldots, q_\ell)$ and $(q_1^\ast, \ldots, q_\ell^\ast)$ differ.
\end{itemize}

According to San Mart{\'\i}n  and Quintana~\cite{martin2002consistency}, identifiability is necessary for the existence
of both asymptotically unbiased estimates and consistent estimates. This observation leaves open the question of
how to find such estimates.
In the following, we strive to design estimators of the class prior probabilities $q_y$ as unique solutions to
systems of linear equations\footnote{%
Other popular approaches to designing estimators include distribution matching 
(Gonz{\`a}lez et al.~\cite{Gonzalez:2017:RQL:3145473.3117807} and the references therein), 
ensemble methods (Serapi{\~a}o et al.~\cite{SerapiaoEnsembles2023} and the references therein)
and expectation maximisation as implementation of maximum likelihood estimation (Saerens et
al.~\cite{saerens2002adjusting}).}.

Calling the following result 
a theorem is an exaggeration as its proof is very short and basic. But it is fundamental for the study
and estimation of prior probability shift and in that sense deserves being called a theorem. Of course,
Theorem~\ref{th:basic} is not novel. In particular \eqref{eq:Firat} was mentioned by Saerens et 
al.~\cite{saerens2002adjusting} (Eq.~(2.5), with $Z$ chosen as a hard classifier) 
and quite likely also in earlier works. Even so, linking 
the notion of prior probability shift to the training dataset covariances of functions of the features 
and the indicators of the classes or the posterior class probabilities might have some degree of novelty,
at least in the multi-class case.

\begin{subequations}
\begin{theorem}\label{th:basic} Let $p_y = P[Y=y]$ and $q_y = Q[Y=y]$ for $y \in \mathcal{Y}$. Suppose that 
$P$ and $Q$ are related through prior probability shift in the sense of Definition~\ref{de:priorShift}
and that the  
random variable $Z$ is integrable both under $P$ and $Q$. Then it holds that\footnote{%
For sets $S$, define the indicator function $\mathbf{1}_S$ by $\mathbf{1}_S(s) = 1$ if $s \in S$ and
$\mathbf{1}_S(s) = 0$ if $s \notin S$.}
\begin{align}
E_Q[Z] & = \sum_{y=1}^\ell q_y\, E_P[Z|Y=y] \label{eq:Firat}\\
    & = \sum_{y=1}^\ell \frac{q_y}{p_y} \,\mathrm{cov}_P\bigl(Z,\, \mathbf{1}_{\{Y=y\}}\bigr)  + E_P[Z].
        \label{eq:matrix}
\end{align}
If $Z$ is $X$-measurable, i.e.~if there is a function $f:\mathcal{X}\to \mathbb{R}$ such that 
$Z = f(X)$, then it also follows that\footnote{%
$P[Y=y|X]$ denotes the posterior probability of $Y=y$ given $X$ in the sense
of general conditional probability as defined, for instance, in Section~33 of 
Billingsley~\cite{billingsley1986probability}.}
\begin{equation}
E_Q[Z]  = \sum_{y=1}^\ell \frac{q_y}{p_y} \,\mathrm{cov}_P\bigl(Z,\, P[Y=y|X]\bigr)  + E_P[Z].
        \label{eq:covmatrix}
\end{equation}
\end{theorem}
\end{subequations}

\begin{proof} The theorem follows from the law of total probability combined with
the definitions of conditional expectation and covariance respectively.\qed
\end{proof}

\eqref{eq:Firat} provides the theoretical basis for Firat's (\cite{firat2016unified}, Section~3.2) 
constrained regression approach for quantification under prior probability shift. Firat's $K$ classes
correspond to the $\ell$ classes of this paper. The $L$ rows of Firat's matrix $\mathbf{X}$ emerge
when \eqref{eq:Firat} is applied to $L$ different variables $Z_1, \ldots, Z_L$.

As noted by Firat, \eqref{eq:Firat}, \eqref{eq:matrix} or \eqref{eq:covmatrix} can be the starting
point for setting up a system of linear equations for estimating the class prior probabilities
$q_y$ under prior probability shift. For instance, the choice $f_y(X) = \mathbf{1}_{C_y}(X)$ as crisp (or hard) 
`one vs.~all' classifier for class $y$, learned on the training dataset only, leads to the `Adjusted Count'
estimation approach used in the popular paper by Lipton et al.~\cite{pmlr-v80-lipton18a} who described it
as `method of moments'.
Observe that for this version of `one vs.~all', there is no problem with changing the type of dataset shift,
in contrast to the issue for combined `one vs.~all' quantifiers noted by Friedman~\cite{friedman2014class} and 
Donyavi et al.~\cite{donyavi2023mc}.

Some questions should be considered when designing a concrete instance
of such a linear equation system for quantification.

\emph{How many equations should be used?} If the number of classes in the model is $\ell = |\mathcal{Y}|$
one might conclude that at least $\ell$ equations are needed in order to obtain a unique solution. 
However, as another consequence
of the law of total probability, the $q_y$ must additionally fulfil the linear equation \eqref{eq:sumOne}.
Hence, in order to achieve uniqueness of the solution, 
at least $\ell$ equations must be set up if \eqref{eq:sumOne} is considered a constraint that is checked once
a solution has been found. Alternatively, if \eqref{eq:sumOne} is to be taken into account at the same time as 
the other equations, for uniqueness as a minimum it suffices to set up $\ell -1$ additional equations 
on the basis of Theorem~\ref{th:basic}. Sticking with $\ell -1$ equations has the advantage of reducing
the number of random variables $Z$ that must be chosen for the equations in
Theorem~\ref{th:basic}.

If more then $\ell$ equations are set up the resulting linear equation system for the $q_y$ is overdetermined
such that in its sample-based versions there might be no exact solution at all. Nonetheless, the overdetermined
case is naturally encountered when distribution-matching algorithms are implemented via binning
of the feature space $\mathcal{X}$ (DF$_{x}$ methods) or of the range of a continuous scoring classifier
(DF$_{y}$ methods), see Firat~\cite{firat2016unified}, Casta{\~n}o et al.~\cite{castano2024quantificationlib}
and the references in the latter paper.
To work around the lack of exact solutions, typically approximate solutions are determined by jointly minimising
the differences between the left-hand and right-hand sides of the equations with respect to some specific
metric like the Euclidean norm or the Hellinger divergence 
(see for instance Casta{\~n}o et al.~\cite{castano2024quantificationlib}).

In the following, we focus on the cases of systems of $\ell$ and $\ell-1$ equations, in the latter case together
with constraint \eqref{eq:sumOne}.

\emph{How should the random variables $Z$ appearing in the equations of Theorem~\ref{th:basic}
be chosen?} A very basic criterion for choosing the variables $Z$ is that it must be possible to compute
them from observations of the features $X$ only. This follows from the fact that on the left-hand sides of 
the equations in Theorem~\ref{th:basic}, the variables $Z$ are integrated under the test distribution $Q$ but
the class labels $Y$ are not observed under $Q$. Hence one must make sure that $Z = f(X)$ for some function $f$.

Among others, the following criteria for selecting such functions $f$ have been considered in the
literature:
\begin{itemize}
\item Reducing the variances of the estimated $q_y$. See Friedman~\cite{friedman2014class} and 
Vaz et al.~\cite{Vaz&Izbicki&Stern2019}
for approaches to the direct minimisation of the variance.
Findings by Vaz et al.~\cite{vaz2017prior} and 
Tasche~\cite{Tasche2021Minimising} suggest that deploying variables $Z$ that are able to separate the classes
with high accuracy also reduces the variances of the class prior estimates.
\item Speed of computation. See for instance Hassan et al.~\cite{hassan2020accurately}.
\end{itemize}
With the exception of Hassan et al.~\cite{hassan2020accurately}, in the literature primarily the choices
$Z = \mathbf{1}_{C_y}$ (hard classifier for one of the classes $y$ in $\mathcal{Y}$) and 
$Z = P[Y=y|X]$ (posterior probability under $P$ for class $y$) have been considered.
Below, we consider Friedman's~\cite{friedman2014class} choices of hard classifiers and $Z = P[Y=y|X]$ in more detail.

\section{Friedman's method}\label{se:method}

Friedman~\cite{friedman2014class} proposed two class distribution estimation methods:
\begin{itemize}
\item He discussed in detail one method (later called `Friedman's method' by 
Schuhmacher et al.~\cite{schumacher2021comparative})
based on a specific choice of hard classifiers both for the binary and multi-class cases. We revisit
Friedman's method in this section.
\item Another method, specified only for the binary case, is based on the variance of the posterior positive class
probability under the training distribution (later called `DeBias' method by 
Casta{\~n}o et al.~\cite{castano2024quantificationlib}). This method, without being named, had been mentioned before by
Tasche~\cite{tasche2014exact} as Corollary~6. We discuss this approach in Remark~\ref{rm:debias} below.
\end{itemize}
First, we consider Friedman's method in the binary case $\ell = 2$. As Friedman himself wrote he was not
the first researcher to think about this method.

\emph{Method Max (Forman~\cite{forman2008quantifying}, Section~2.2).} Forman wrote on page~173:
``Considering the earlier discussion of small denominators,
another likely policy is where the denominator \emph{(tpr-fpr)} is maximized: \emph{method Max}.''
Here, Forman referred to crisp binary classifiers (not necessary most accurate) which were
derived from a `raw classifier' (i.e.\ a real-valued scoring classifier).

Accordingly, Friedman's method in the binary case is the special case of Forman's method Max when
the underlying scoring classifier is chosen as the Bayes classifier, i.e.\ the posterior probability of
the positive class.

\emph{Derivation of Friedman's method.} Firat~\cite{firat2016unified} describes 
on p.~2 the rationale for Friedman's method as follows: ``Friedman uses the optimum threshold 
that minimizes the variance of proportion estimates (Friedman,
2014).'' This statement is somewhat misleading, as Friedman~\cite{friedman2014class} actually
does not maximise the variance of the estimator but only the denominator on the right-hand side of 
the following equation (in the notation of this paper)
\begin{equation}\label{eq:binary}
q_1 \ = \ \frac{E_Q[Z]  - E_P[Z|y=2]}{E_P[Z|y=1] - E_P[Z|y=2]},
\end{equation}
over all random variables $0 \le Z=f(X) \le 1$.  
Note that \eqref{eq:binary} is a special case 
of \eqref{eq:Firat} for $\ell = 2$. 

It turns out that 
\begin{equation}\label{eq:max}
\arg\max\limits_{f:\mathcal{X} \to [0,1]} E_P[f(X)|y=1] - E_P[f(X)|y=2] \ = \
f^\ast
\end{equation}
with $f^\ast(x) = 1$ if $P[Y=1|X=x] > p_1$, $f^\ast(x) = 0$ if $P[Y=1|X=x] < p_1$ and
$f^\ast(x)$ arbitrary if $P[Y=1|X=x] = p_1$. 

The solution to the problem of minimising the sample variance of the estimator defined by
\eqref{eq:binary} is less obvious. It has been tackled numerically 
by Vaz et al.~(\cite{Vaz&Izbicki&Stern2019}, Section~2.3), and by Tian et al.~\cite{TianICML2023} by involving
influence functions.

\begin{remark}\label{rm:alternative}
Friedman~\cite{friedman2014class} and subsequent users of his method appear to have implemented it
by means of  plugging-in an estimate of the posterior probability $P[Y=1|X]$ into the function $f^\ast$ as defined
in \eqref{eq:max}. However, as $P[Y=1|X]$ could be difficult to estimate with satisfactory accuracy, such
an implementation might be suboptimal.

Note that \eqref{eq:max} is equivalent to
\begin{equation}\label{eq:min}
\arg\min\limits_{f:\mathcal{X} \to [0,1]} (1-p_1)\,E_P[f(X)\,\mathbf{1}_{\{Y=1\}}] + 
	p_1\,E_P[(1-f(X))\,\mathbf{1}_{\{Y=2\}}] \ = \ 1 - f^\ast,
\end{equation}
with $f^\ast$ as in \eqref{eq:max}. \eqref{eq:min} can be read as the problem to minimise the expected cost-sensitive 
error for a binary classification problem. This problem can be dealt with directly through a variety of 
approaches, resulting in approximations of the optimal classifier which do not require the estimation of 
$P[Y=1|X]$. The cost-sensitive minimisation problem can also be translated into a standard classification
problem by appropriate re-weighting (Zadrozny et al.~\cite{zadrozny2003weighting}).\hfill \qed
\end{remark}

\emph{Friedman's method for more than two classes.} Friedman~\cite{friedman2014class} suggested
defining $Z_y = f^\ast_y(X)$ for $y \in \mathcal{Y}$ with $f^\ast_y(x) = 1$ if $P[Y=y|X=x] > p_y$,
 $f^\ast_y(x) = 0$ if $P[Y=y|X=x] \le p_y$, and then using \eqref{eq:Firat} with $Z_y$, $y = 1, \ldots, \ell$,
to obtain a system of $\ell$ linear equations for the test prior probabilities of the classes $y \in \mathcal{Y}$.

According to Schuhmacher et al.~\cite{schumacher2021comparative}, Friedman's method works well in 
binary quantification problems and still achieves good performance in multi-class settings.

\section{Uniqueness of solutions and covariance matrix-based approaches}
\label{se:unique}

As discussed in Section~\ref{se:linear}, uniqueness of the solutions is an important design criterion 
for setting up a system of linear equations for class distribution estimation under prior probability shift.
In this section, we provide more detail regarding the number of equations needed and look closer at designs
based on covariance matrices estimated in the training dataset.

\subsection{How many equations are needed?}

\eqref{eq:covmatrix} of Theorem~\ref{th:basic} is interesting because the choice $Z = P[Y=y|X]$ for
fixed $y = 1, \ldots, \ell$, implies the matrix identity
\begin{gather}
\begin{pmatrix}
E_Q\bigl[P[Y=1|X]\bigr] - p_1\\
\vdots\\
E_Q\bigl[P[Y=\ell|X]\bigr] - p_\ell
\end{pmatrix}
 =
\Sigma_P
\times
\begin{pmatrix}
\frac{q_1}{p_1}\\
\vdots\\
\frac{q_\ell}{p_\ell}
\end{pmatrix}, \label{eq:covarianceMatrix}\\
\Sigma_P  =
\begin{pmatrix}
\mathrm{cov}_P\bigl(P[Y=1|X],\, P[Y=1|X]\bigr) & \ldots & 
    \mathrm{cov}_P\bigl(P[Y=1|X],\, P[Y=\ell|X]\bigr)\\
\vdots & \ddots & \vdots \\
\mathrm{cov}_P\bigl(P[Y=\ell|X],\, P[Y=1|X]\bigr) & \ldots & 
    \mathrm{cov}_P\bigl(P[Y=\ell|X],\, P[Y=\ell|X]\bigr)
\end{pmatrix}.\notag
\end{gather}
\eqref{eq:covarianceMatrix} connects the prior class probabilities $p_y$ under the training distribution,
the prior class probabilities $q_y$ under the test distribution, and the averages under the test 
distribution $E_Q\bigl[P[Y=y|X]\bigr]$ of the training posterior class probabilities through the covariance matrix 
$\Sigma_P$ of the training posterior probabilities under the training distribution.
All quantities in \eqref{eq:covarianceMatrix} but the test class prior probabilities $q_y$ can be estimated
from the training dataset and the features in the test dataset in principle. 
Hence, if the square matrix $\Sigma_P$ were invertible, 
\eqref{eq:covarianceMatrix} could be solved for the $q_y$ by matrix inversion.

Unfortunately, as follows from the following proposition, the covariance matrix $\Sigma_P$ is never invertible.

\begin{proposition}\label{pr:low}
Let $Z_1, \ldots, Z_r$ be integrable random variables under the distribution $P$. Suppose that
$Y$ is a discrete random variable with values in $\mathcal{Y} = \{1, \ldots, \ell\}$ with $\ell \ge 2$ and
$X$ is a random vector with values in $\mathcal{X}$. Define the matrices
$M = (m_{ij})_{\substack{i=1, \ldots, r\\ j=1, \ldots, \ell}}$ and 
$M^\ast = (m^\ast_{ij})_{\substack{i=1, \ldots, r\\ j=1, \ldots, \ell}}$ by
\begin{equation*}
m_{ij} = \mathrm{cov}(Z_i, \,\mathbf{1}_{\{Y=j\}}) \quad \text{and} \quad
m^\ast_{ij} = \mathrm{cov}\bigl(Z_i, \,P[Y=j|X]\bigr).
\end{equation*}
Then it follows that 
\begin{equation*}
\mathrm{rank}(M) \le \ell -1 \quad \text{and} \quad \mathrm{rank}(M^\ast) \le \ell -1.
\end{equation*}
\end{proposition}

\begin{proof}
Due to the fact that $1 = \sum_{j=1}^\ell \mathbf{1}_{\{Y=j\}}$ and $1 = \sum_{j=1}^\ell P[Y=j|X]$,
the vector $v = (1, 1, \ldots, 1)^T \in \mathbb{R}^\ell$ is an element of the kernels of $M$ and $M^\ast$, i.e.\
it holds that $M\times v = 0 = M^\ast \times v$. This implies the assertion. \hfill \qed
\end{proof}

As a consequence of Proposition~\ref{pr:low}, there is no possible choice of random variables 
$Z_1, \ldots, Z_\ell$ that could serve on the basis of \eqref{eq:matrix} or \eqref{eq:covmatrix}
to create a system of $\ell$ linear equations with a unique solution 
for the $\ell$ unknowns $q_1, \ldots, q_\ell$. However, Proposition~\ref{pr:low} leaves open the 
question if such an equation system can be constructed on the basis of \eqref{eq:Firat}. 

\begin{remark}\label{rm:consistent}
For integrable random variables $Z_1, \ldots, Z_r$, define the matrix 
$\widetilde{M} = (\widetilde{m}_{i j})_{\substack{i=1, \ldots, r\\ j=1, \ldots, \ell}}$ by
$\widetilde{m}_{i j} \ = \ E_P[Z_i|Y=j]$.

\begin{subequations}
$\widetilde{M}$ can be rewritten as
\begin{equation}\label{eq:Firat.1}
\widetilde{M}\ =\  L \times D,
\end{equation}
where 
\begin{equation}\label{eq:consistent}
L \ =\  \begin{pmatrix} 
E_P[Z_1\,\mathbf{1}_{\{Y=1\}}] & \ldots & E_P[Z_1\,\mathbf{1}_{\{Y=\ell\}}]\\
\vdots & \ddots & \vdots \\
E_P[Z_r\,\mathbf{1}_{\{Y=1\}}] & \ldots & E_P[Z_r\,\mathbf{1}_{\{Y=\ell\}}]
\end{pmatrix}
\end{equation}
\end{subequations}
and 
$D = (d_{i j})_{i,j = 1, \ldots, \ell}$ is the diagonal matrix with
$d_{i j} = \frac{1}{p_i}$ if $i = j$ and $d_{i j} = 0$ if $i \neq j$. In particular, we have
$\mathrm{rank}(D) = \ell$.

Define the vector $v = (1, 1, \ldots, 1)^T$ as in the
proof of Proposition~\ref{pr:low}. Then it follows that $L \times v = (E_P[Z_1], \ldots, E_P[Z_r])^T$.
If $Z_1, \ldots, Z_r$ are chosen such that $(E_P[Z_1], \ldots, E_P[Z_r]) \neq 0$, as a consequence 
$L \times v \neq 0$ results. Hence there is no obvious reason as in the case of Proposition~\ref{pr:low}
for the rank of $L$ (and by \eqref{eq:Firat.1} also of $\widetilde{M}$) to be less than maximal,
i.e.~being equal to $\min(r,\ell)$. This observation suggests that \eqref{eq:Firat} can be used
to obtain a system of $\ell$ linear equations with a unique solution for the
test class prior probabilities $q_1, \ldots, q_\ell$.\hfill \qed
\end{remark}

\subsection{Invertible covariance matrices}
\label{se:invertible}

The fact that the covariance $\Sigma_P$ of the posterior class probabilities
$P[Y=y|X]$, $y \in \mathcal{Y}$ in \eqref{eq:covarianceMatrix} cannot be inverted 
is caused by the linear dependence between the posterior probabilities since
$\sum_{y=1}^\ell P[Y=y|X]  =  1$.
This issue can be avoided by disregarding one of probabilities, say $P[Y=\ell|X]$. Indeed,
making use of the identity $\mathbf{1}_{\{Y=\ell\}} = 1 - \sum_{y=1}^{\ell-1} \mathbf{1}_{\{Y=y\}}$ 
in \eqref{eq:matrix} produces the following corollary to Theorem~\ref{th:basic}.

\begin{subequations}
\begin{corollary}\label{co:invert}
Let $p_y = P[Y=y]$ and $q_y = Q[Y=y]$ for $y \in \mathcal{Y}$. Suppose that 
$P$ and $Q$ are related through prior probability shift in the sense of Definition~\ref{de:priorShift}
and that the  
random variable $Z$ is integrable both under $P$ and $Q$. Then it holds that
\begin{equation}
E_Q[Z] = \sum_{y=1}^{\ell-1} \left(\frac{q_y}{p_y}  - \frac{q_\ell}{p_\ell}\right) 
    \mathrm{cov}_P\bigl(Z,\, \mathbf{1}_{\{Y=y\}}\bigr)  + E_P[Z].
        \label{eq:invmatrix}
\end{equation}
If $Z$ is $X$-measurable, i.e.~if there is a function $f:\mathcal{X}\to \mathbb{R}$ such that 
$Z = f(X)$, then it also follows that
\begin{equation}
E_Q[Z]  = \sum_{y=1}^{\ell-1} \left(\frac{q_y}{p_y}  - \frac{q_\ell}{p_\ell}\right) 
        \mathrm{cov}_P\bigl(Z,\, P[Y=y|X]\bigr)  + E_P[Z].
        \label{eq:invcovmatrix}
\end{equation}
\end{corollary}
\end{subequations}

Corollary~\ref{co:invert} suggests the following approach to estimating the test
class prior probabilities $q_1, \ldots, q_\ell$.
\begin{subequations}
\begin{corollary}\label{co:solved}
Assume that the functions $f_1, \ldots, f_{\ell-1}: X \to \mathbb{R}$ are such that the matrix
\begin{equation}\label{eq:C}
C \ = \ \begin{pmatrix}
\mathrm{cov}(f_1(X),\, \mathbf{1}_{\{Y=1\}}) & \ldots & \mathrm{cov}(f_1(X),\, \mathbf{1}_{\{Y=\ell-1\}}) \\
\vdots & \ddots & \vdots \\
\mathrm{cov}(f_{\ell-1}(X),\, \mathbf{1}_{\{Y=1\}}) & \ldots & \mathrm{cov}(f_{\ell-1}(X),\, \mathbf{1}_{\{Y=\ell-1\}})
\end{pmatrix}
\end{equation}
has rank $\ell-1$, i.e.~it is invertible.

Let $\bigl(E_Q[f_1(X)]-E_P[f_1(X)], \ldots, E_Q[f_{\ell-1}(X)]-E_P[f_{\ell-1}(X)]\bigr)^T = z$
and $C^{-1}\times z = (s_1, \ldots, s_{\ell-1})^T$.

Then it follows that 
\begin{equation}\label{eq:q_y}
q_y  = p_y \left(s_y + 1 - \sum_{i=1}^{\ell-1} p_i\,s_i\right), \ y = 1, \ldots, \ell-1,\quad 
q_\ell  = p_\ell \left(1 - \sum_{i=1}^{\ell-1} p_i\,s_i\right).
\end{equation}
\end{corollary}
\end{subequations}

Observe that as a consequence of the general properties of conditional expectation\footnote{%
See, for instance, Problem~34.6 of Billingsley~\cite{billingsley1986probability}).} matrix $C$
of \eqref{eq:C} can be represented as
\begin{equation}\label{eq:Calt}
C \ = \ \begin{pmatrix}
\mathrm{cov}(f_1(X),\, P[Y=1|X]) & \ldots & \mathrm{cov}(f_1(X),\, P[Y=\ell-1|X]) \\
\vdots & \ddots & \vdots \\
\mathrm{cov}(f_{\ell-1}(X),\, P[Y=1|X]) & \ldots & \mathrm{cov}(f_{\ell-1}(X),\, P[Y=\ell-1|X])
\end{pmatrix}.
\end{equation}
With the special choice $f_y(X) = P[Y=y|X]$ for $y = 1, \ldots, \ell-1$ matrix $C$ as represented in
\eqref{eq:Calt} becomes the
covariance matrix of $\Sigma^\ast_P$ of $P[Y=1|X], \ldots, P[Y=\ell-1|X]$.

\begin{subequations}
\begin{remark}[DeBias method]\label{rm:debias}
Suppose we are in the binary case $\ell =2$ and apply Corollary~\ref{co:solved} with $C$ as given in 
\eqref{eq:Calt} and $f_1(X) = P[Y=1|X]$. This implies $C = \Sigma^\ast_P
= \mathrm{var}\bigl[P[Y=1|X]\bigr]$. We then obtain by means of \eqref{eq:q_y}
\begin{align}
q_1  & =  \frac{p_1\,(1-p_1)}{\mathrm{var}_P\bigl[P[Y=1|X]\bigr]}\,\bigl( E_Q\bigl[P[Y=1|X]\bigr]-p_1\bigr) +
    p_1,\label{eq:debias.solved}\\
\intertext{which is equivalent to }
E_Q\bigl[P[Y=1|X]\bigr] & = q_1\,\frac{\mathrm{var}_P\bigl[P[Y=1|X]\bigr]}{p_1\,(1-p_1)} +
    p_1 \left(1 - \frac{\mathrm{var}_P\bigl[P[Y=1|X]\bigr]}{p_1\,(1-p_1)}\right).\label{eq:debias}
\end{align}
\eqref{eq:debias} appears to have been first published by Tasche~\cite{tasche2014exact} (Corollary~6) and
then to have been presented at a conference by Friedman~\cite{friedman2014class}. This approach to estimating
$q_1$ has been called `DeBias' method by Casta{\~n}o et al.~\cite{castano2024quantificationlib}.

Hence, Corollary~\ref{co:solved} with $C = \Sigma^\ast_P$ may be interpreted as multi-class extension of the
DeBias approach.\hfill \qed
\end{remark}
\end{subequations}

\begin{subequations}
\begin{remark}[Probabilistic Adjusted Count (PAC)]\label{rm:PAC}
Suppose again we are in the binary case $\ell =2$ and apply Corollary~\ref{co:solved}, this time with
$C$ as represented in \eqref{eq:C} and $f_1(X) = P[Y=1|X]$.
This implies $C = E_P[\bigl[P[Y=1|X]\,\mathbf{1}_{\{Y=1\}}\bigr] -p_1^2$. We then obtain by means of \eqref{eq:q_y}
\begin{align}
q_1  & =  p_1\,(1-p_1) \frac{E_Q\bigl[P[Y=1|X]\bigr]-p_1}{E_P[\bigl[P[Y=1|X]\,\mathbf{1}_{\{Y=1\}}\bigr] -p_1^2} + 
    p_1\,\label{eq:PAC.1}\\
\intertext{which is equivalent to }
q_1  & = \frac{E_Q\bigl[P[Y=1|X]\bigr] - E_P\bigl[P[Y=1|X] \bigm| Y=2 \bigr]}
    {E_P\bigl[P[Y=1|X] \bigm| Y=1 \bigr] - E_P\bigl[P[Y=1|X] \bigm| Y=2 \bigr]}.\label{eq:PAC}
\end{align}
\eqref{eq:PAC} was called `probability
estimation \& average (P\&A)' method by Bella et al.~\cite{bella2010quantification} but 
is now commonly referred to as 
`probabilistic adjusted count (PAC)' (Gonz\'alez et al.~\cite{Gonzalez:2017:RQL:3145473.3117807}). Its multi-class
extension is sometimes called `generalized probabilistic adjusted count (GPAC)' (see, for instance, Schuhmacher
et al.~\cite{schumacher2021comparative}) and also covered by Corollary~\ref{co:solved} with
the choice $f_y(X) = P[Y=y|X]$ in \eqref{eq:C}. \hfill \qed
\end{remark}
\end{subequations}

\begin{remark}\label{rm:same}
Observe that in \eqref{eq:PAC.1} it holds that 
\begin{equation*}
E_P[\bigl[P[Y=1|X]\,\mathbf{1}_{\{Y=1\}}\bigr] -p_1^2 \ = \ \mathrm{var}_P\bigl[P[Y=1|X]\bigr].
\end{equation*}
By \eqref{eq:debias.solved}, therefore in the binary case the DeBias and PAC methods for class 
distribution estimation are identical at population level, i.e.~with infinite training and test
datasets. This observation is not necessarily true 
in practice when DeBias and PAC estimates respectively are calculated based on sample versions
of \eqref{eq:debias.solved} and \eqref{eq:PAC.1}. \hfill \qed
\end{remark}

\section{Comparing asymptotic variances}\label{se:example}

As mentioned in Section~\ref{se:setting}, we consider class distribution estimation as a two-sample problem:
\begin{itemize}
\item A training sample for estimating certain quantities (e.g.~the true positive and false negative rates of
a classifier) under the training distribution because the quantities are needed for 
estimating the class prior probabilities under the test distribution.
\item A test sample for estimating the class prior probabilities under the test distribution, based on the
quantities estimated on the training sample.
\end{itemize}
Hence minimising the error of a method for class distribution estimation means minimising the 
estimation errors on the two samples.

In the following, we look at the semi-asymptotic binary case ($\ell = 2$) where
\begin{itemize}
\item the training distribution $P$ is known (infinite sample) such that
the prior class probabilities $p_y$ and the posterior class probabilities $P[Y=y|X]$ can be
exactly determined in the sense that the estimation error on the training sample vanishes.
\item From the test distribution a finite but large sample of size $n$ is given, and we focus
upon unbiased estimators of the class prior probabilities.
\end{itemize}
For unbiased estimators the Cram{\'e}r-Rao lower bound specifies a minimum value for the variance 
that cannot be undercut.
Denote by $\widehat{q}_n^{\text{ML}}$ the maximum-likelihood (ML) estimator 
of the test prior probability $q_1$ of class $1$ and by $\sigma^2_{\text{ML}}$ 
its so-called asymptotic variance. Then 
$\frac{\sigma_{\text{ML}}^2}{n}$ is the Cram{\'e}r-Rao lower bound for the variances
of the unbiased estimators of $q_1$ on test samples of size $n$ when the training distribution is known 
(called here `asymptotic setting'), see Section~5 of Tasche~\cite{Tasche2021Minimising}.

We compare $\sigma^2_{\text{ML}}$  with the asymptotic variances in the sense of Definition 10.1.9 of 
Casella and Berger~\cite{Casella&Berger} of the Friedman estimator $\widehat{q}_n^{\text{Fried}}$ and 
the DeBias estimator $\widehat{q}_n^{\text{DeBias}}$ of the test prior probability $q_1$ of class $1$.

We assume that both the conditional distribution of $X$ given $Y=1$ and the conditional distribution of
$X$ given $Y=2$ have densities $g_1>0$ and $g_2>0$ with respect to some measure\footnote{%
In Example~\ref{ex:binormal} below $\mu$ is the Lebesgue measure on $\mathbb{R}$.
} $\mu$. In particular, then the posterior probability
$P[Y=1|X=x]$ can be represented as
\begin{equation}\label{eq:posterior}
P[Y=1|X=x] \ =\ \frac{p_1\,g_1(x)}{p_1\,g_1(x) + (1-p_1)\,g_2(x)},
\end{equation}
and the density of the feature vector $X$ under the test distribution $Q$ is given by
\begin{equation}
g_Q \ = \ q_1\,g_1 + (1-q_1)\,g_2.
\end{equation}

Since the training distribution $P$ is assumed to be known, in the following all expected values $E_P[Z]$ 
are deterministic values that need not be estimated. In particular, also the prior probabilities
$p_1$ and $p_2 = 1 - p_1$ are known constants. In contrast, the test distribution $Q$ is not known but an i.i.d.~sample
$X_1, \ldots, X_n$ of the feature vector $X$ drawn from its distribution under $Q$ is observed.

\emph{ML estimator.} For a detailed description of the ML estimator $\widehat{q}_n^{\text{ML}}(X_1, \ldots, X_n)$
$= \widehat{q}_n^{\text{ML}}$
we refer to Section~4 of Tasche~\cite{Tasche2021Minimising}, as there is no closed-form representation of
the ML estimator. However, its asymptotic variance $\sigma_{\text{ML}}$ under $Q$ is known:
\begin{equation}
\sigma_{\text{ML}}^2  =  E_Q\left[\Big(
    \frac{g_1(X) - g_2(X)}{g_Q(X)}\Big)^2 \right]^{-1}
     = \ \frac{q_1^2\,(1-q_1)^2}{\mathrm{var}_Q\bigl[E_Q[Y=1|X]\bigr]}.\label{eq:bound}
\end{equation}
$\sigma_{\text{ML}}^2$ is characterised through 
the property that $\sqrt{n}\,\bigl(\widehat{q}_n^{\text{ML}} - q_1\bigr)$ 
converges in distribution toward the normal distribution with mean $0$ and variance $\sigma_{\text{ML}}^2$.
Observe that $\sigma_{\text{ML}}^2$  is a function of $q_1$ but not of $p_1$.

\emph{Friedman estimator.} In the binary case, under the assumption on 
semi-asympto{\-}tics made for this section, the Friedman
estimator  $\widehat{q}_n^{\text{Fried}}(X_1, \ldots, X_n)$ $=\widehat{q}_n^{\text{Fried}}$  based on the
homonymous method presented in Section~\ref{se:method} can be written as
\begin{subequations}
\begin{equation}\label{eq:Fried.est}
\widehat{q}_n^{\text{Fried}} \ = \ \frac{\frac{1}{n} \sum_{i=1}^n f^\ast(X_i)  - E_P[f^\ast(X)|y=2]}
	{E_P[f^\ast(X)|y=1] - E_P[f^\ast(X)|y=2]},
\end{equation}
with $f^\ast$ defined through \eqref{eq:max}. 
Friedman~\cite{friedman2014class} observed that $f^\ast$ can also be represented as
\begin{equation}\label{eq:Nop}
f^\ast(x)\  =\ \begin{cases}
1, & \mathrm{if\ }g_1(x) > g_2(x),\\
0, & \mathrm{if\ }g_1(x) \le g_2(x).
\end{cases}
\end{equation}
As a consequence of \eqref{eq:Nop}, the right-hand side of \eqref{eq:Fried.est} does not depend
on $p_1$ or $p_2$ for $f^\ast(X)$ or any of the $f^\ast(X_i)$. Therefore, also $\widehat{q}_n^{\text{Fried}}$ as
defined in \eqref{eq:Fried.est} does not change if $p_1$ or $p_2$ are changed.
From the central limit theorem, it follows that $\sqrt{n}\,(\widehat{q}_n^{\text{Fried}} - q_1)$ under
$Q$ converges toward a normal distribution with mean $0$ and variance 
$\sigma_{\text{Fried}}^2$.
More precisely, the asymptotic variance of $\widehat{q}_n^{\text{Fried}}$ is
\begin{equation}
\sigma_{\text{Fried}}^2\ = \ \frac{E_Q[f^\ast(X)]\,\bigl(1 - E_Q[f^\ast(X)]\bigr)}
	{(E_P[f^\ast(X)|y=1] - E_P[f^\ast(X)|y=2])^2}.
\end{equation}
\end{subequations}

\emph{DeBias estimator.} In the binary case, under the assumption on semi-asymp{\-}totics 
made for this section, the DeBias
estimator $\widehat{q}_n^{\text{DeBias}} = \widehat{q}_n^{\text{DeBias}}(X_1, \ldots, X_n)$ based on the
method presented in Remark~\ref{rm:debias} can be written as
\begin{subequations}
\begin{equation}\label{eq:DeBias.est}
\widehat{q}_n^{\text{DeBias}} \ = \ 
	\frac{p_1\,(1-p_1)}{\mathrm{var}_P\bigl[P[Y=1|X]\bigr]}\,\left( \frac{1}{n} \sum_{i=1}^n P[Y=1|X=X_i]-p_1\right) +
    p_1.
\end{equation}
From the central limit theorem, it follows that $\sqrt{n}\,(\widehat{q}_n^{\text{DeBias}} - q_1)$ under
$Q$ converges toward the normal distribution with mean $0$ and variance 
$\sigma_{\text{DeBias}}^2$,
or more precisely, the asymptotic variance of $\widehat{q}_n^{\text{DeBias}}$ is
\begin{equation}\label{eq:sigma.DeBias}
\sigma_{\text{DeBias}}^2\ = \ \left(\frac{p_1\,(1-p_1)}{\mathrm{var}_P\bigl[P[Y=1|X]\bigr]}\right)^2\,
	\mathrm{var}_Q\bigl[P[Y=1|X]\bigr].
\end{equation}
\end{subequations}
Note that it follows from \eqref{eq:bound} and \eqref{eq:sigma.DeBias} that 
$\sigma_{\text{ML}}^2 = \sigma_{\text{DeBias}}^2$ in the case of no shift, i.e.~$p_1=q_1$.
As all quantities derived from $P$ are assumed to be constant in the setting of this section, it
follows as in Remark~\ref{rm:same} that the asymptotic variance $\sigma_{\text{PAC}}^2$ of the
PAC estimator discussed in Remark~\ref{rm:PAC} is identical with $\sigma_{\text{DeBias}}^2$, i.e.
$\sigma_{\text{DeBias}}^2  =  \sigma_{\text{PAC}}^2$.
For this reason, PAC is omitted from the following numerical example.

\begin{figure}[ht]
  \centering
  \includegraphics[width=\linewidth]{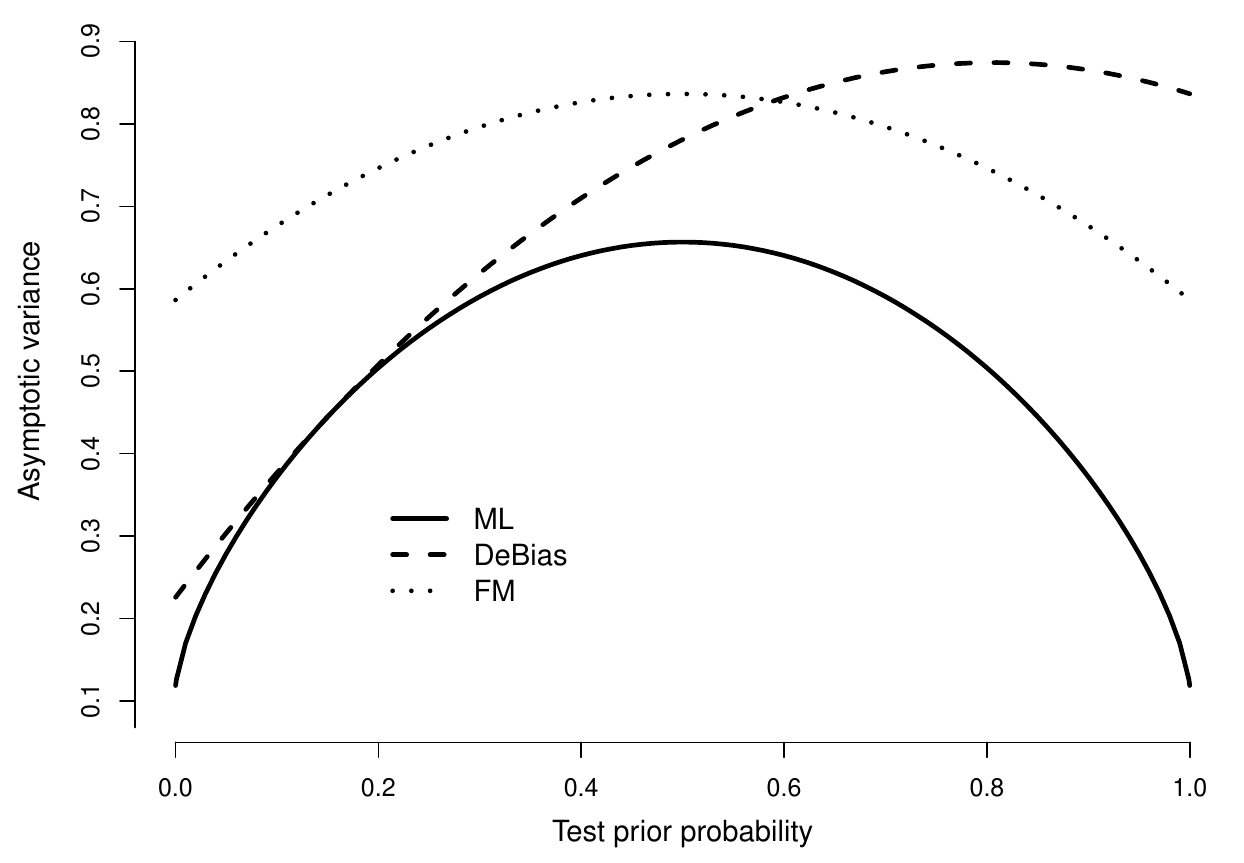}
  \caption{Asymptotic variances of maximum likelihood estimator, DeBias estimator and
  Friedman estimator in a binormal model. See Example~\ref{ex:binormal} for the specification of the underlying
  model.}\label{fig:1}
\end{figure}

\begin{example}\label{ex:binormal}
We consider the same univariate binormal model with equal variances of the class-conditional
distributions as in Section~7 of Tasche~\cite{Tasche2021Minimising}: 
The two normal class-conditional distributions of the feature variable $X$ are given by
\begin{subequations}
\begin{equation}\label{eq:binorm}
X\,|\,Y=i \ \sim \ \mathcal{N}(\mu_i, \sigma^2), \qquad i=1,2
\end{equation}
for conditional means $\mu_2 < \mu_1$ and some $\sigma > 0$. We choose 
\begin{equation}\label{eq:spec}
\mu_1 = 1.5, \quad \mu_2 = 0, \quad\text{and}\quad \sigma = 1.
\end{equation}
The model is then completely specified by choosing $p_1 = 0.15$ for the training prior probability of class~$1$.
The test prior probability $q_1$ of class~$1$ is not fixed as we calculate asymptotic variances of the
three above-mentioned prior distribution estimators for the whole range $(0,1)$ of $q_1$. The results are 
shown in Figure~\ref{fig:1}.\hfill \qed
\end{subequations}
\end{example}

The following observations can be made from Figure~\ref{fig:1}:
\begin{itemize}
\item The asymptotic variance of the ML estimator is uniformly lower than the asymptotic variances of the other
estimators for the whole possible range of the test prior probability of class~$1$ as is to be expected as a
consequence of the Cram{\'e}r-Rao inequality.
\item The asymptotic variance of the Friedman estimator is not uniformly lower 
than the asymptotic variance of the DeBias estimator 
and vice versa.
\item The DeBias estimator is almost optimal in the vicinity of the training 
prior probability ($p_1 = 0.15$) of class~$1$, as a consequence of \eqref{eq:bound} and \eqref{eq:sigma.DeBias}.
\item In contrast, the asymptotic variance of the DeBias estimator is much larger than the asymptotic variance of the
Friedman estimator in the $(0.8, 1)$ range of the test prior probability that is 
far away from the training prior probability $0.15$.
\end{itemize}

\section{Conclusions}
\label{se:conclusions}

We have considered Friedman's~\cite{friedman2014class} method in the 
context of a general framework for designing linear equation systems
for class distribution estimation and compared its binary version with DeBias which is another
method proposed by Friedman, and the maximum likelihood estimator. The main findings
of this paper are the following:
\begin{itemize}
\item The population versions of DeBias and Probability Adjusted Count 
(PAC, Bella et al.~\cite{bella2010quantification})
are identical and the binary special case of a new estimation approach based on inverting the covariance matrix
of the training posterior class probabilities (see Section~\ref{se:invertible}).
\item Although the definition of Friedman's method appears to involve evaluations 
of the posterior probabilities under the training distribution, the method is potentially
less sensitive to inaccuracies of the posterior estimates on smaller training datasets than
the maximum likelihood estimator. This is a consequence of the fact that Friedman's methode can be implemented  
without a need to estimate the training posterior class probabilities (see Section~\ref{se:method}).
\item As shown in Example~\ref{ex:binormal}, Friedman's method may be locally outperformed in terms of asymptotic 
variance by DeBias. But thanks to its independence of the training prior class probabilities its performance
is relatively uniform over the full range of possible values of the test prior probability of 
the positive class (class $1$ in Example~\ref{ex:binormal}), in contrast to DeBias' poor performance for test 
prior probabilities which are very different 
to the corresponding training prior probability.
\end{itemize}

\begin{credits}
\subsubsection{\ackname} 
The author would like to thank three anonymous reviewers
for their useful comments and suggestions.

\end{credits}
%
%
%
\bibliographystyle{splncs04}
\bibliography{C:/Users/Dirk/Documents/LehreForschung/Papers/Literature}

\addcontentsline{toc}{section}{References}

\end{document}